\documentclass{article}
\usepackage{ijcai20}

\usepackage{times}
\usepackage{soul}
\usepackage{url}
\usepackage[hidelinks]{hyperref}
\usepackage[utf8]{inputenc}
\usepackage[small]{caption}
\usepackage{graphicx}
\usepackage{amsmath}
\usepackage{booktabs}
\usepackage{algorithm}
\usepackage{algorithmic}
\usepackage{amsthm}
\usepackage{amssymb}
\usepackage{subfigure}
\usepackage[american]{babel}
\usepackage{footnote}
\usepackage{xcolor}
{
	\newtheorem{corollary}{Corollary}
	\newtheorem{theorem}{Theorem}
}
\usepackage{mathtools}

\title{TRP:\ Trained Rank Pruning for Efficient Deep Neural Networks}

\author{
	Yuhui~Xu$^1$
	\and
	Yuxi~Li$^1$\and
	Shuai~Zhang$^2$\and
	Wei Wen$^3$\and
	Botao~Wang$^2$\and\\
	Yingyong~Qi$^2$\and
	Yiran Chen$^3$\and
	Weiyao~Lin$^1$\and
	Hongkai~Xiong$^1$
	\affiliations
	$^1$Shanghai Jiao Tong University\quad
	$^2$Qualcomm AI Research\quad
	$^3$Duke University
	\emails
	\{yuhuixu, lyxok1, wylin, xionghongkai\}@sjtu.edu.cn,\\
	\{shuazhan, botaow, yingyong\}@qti.qualcomm.com,
	\{wei.wen, yiran.chen\}@duke.edu
}

\begin{document}
	
	\maketitle
	
	\begin{abstract}
	    \let\thefootnote\relax\footnote{This work was supported in part by the National Natural Science Foundation of China under Grants 61720106001, 61932022, and in part by the Program of Shanghai Academic Research Leader under Grant 17XD1401900.}
	    
		To enable DNNs on edge devices like mobile phones, low-rank approximation has been widely adopted because of its solid theoretical rationale and efficient implementations. Several previous works attempted to directly approximate a pre-trained model by low-rank decomposition; however, small approximation errors in parameters can ripple over a large prediction loss. As a result, performance usually drops significantly and a sophisticated effort on fine-tuning is required to recover accuracy.
		Apparently, it is not optimal to separate low-rank approximation from training. Unlike previous works, this paper integrates low rank approximation and regularization into the training process. We propose Trained Rank Pruning (TRP), which alternates between low rank approximation and training. TRP maintains the capacity of the original network while imposing low-rank constraints during training. A nuclear regularization optimized by stochastic sub-gradient descent is utilized to further promote low rank in TRP. The TRP trained network inherently has a low-rank structure, and is approximated with negligible performance loss, thus eliminating the fine-tuning process after low rank decomposition. The proposed method is comprehensively evaluated on CIFAR-10 and ImageNet, outperforming previous compression methods using low rank approximation.
	\end{abstract}
	
	\section{Introduction}
	
	Deep Neural Networks (DNNs) have shown remarkable success in many computer vision tasks. Despite the high performance in server-based DNNs powered by cutting-edge parallel computing hardware, most state-of-the-art architectures are not yet ready to be deployed on mobile devices due to the limitations on computational capacity, memory and power. 
	
	To address this problem, many network compression and acceleration methods have been proposed. Pruning based methods \cite{han2015learning,He_2017_ICCV,Liu2017learning,Luo2017ThiNetAF} explore the sparsity in weights and filters. Quantization based methods \cite{han2015learning,Zhou2016Incremental,Courbariaux2016BinaryNet,rastegari2016xnor,Xu2018DeepNN} reduce the bit-width of network parameters. Low-rank decomposition \cite{Denton2014Exploiting,jaderberg2014speeding,Guo2018Network,Wen2017Coordinating,Alvarez2017Compression} minimizes the channel-wise and spatial redundancy by decomposing the original network into a compact one with low-rank layers. In addition, efficient architectures \cite{Sandler2018MobileNetV2,Ma2018ShuffleNet} are carefully designed to facilitate mobile deployment of deep neural networks. Different from precedent works, this paper proposes a novel approach to design low-rank networks.
	
	Low-rank networks can be trained directly from scratch. However, it is difficult to obtain satisfactory results for several reasons. 
	(1)~\textit{Low capacity:}
	compared with the original full rank network, the capacity of a low-rank network is limited, which causes difficulties in optimizing its performances. 
	(2)~\textit{Deep structure:}
	low-rank decomposition typically doubles the number of layers in a network. The additional layers make numerical optimization much more vulnerable to gradients explosion and/or vanishing. 
	(3)~\textit{Heuristic rank selection:}
	the rank of decomposed network is often  chosen as a hyperparameter based on pre-trained
	networks; this may not be the optimal rank for the network trained from scratch.
	
	Alternatively, several previous works \cite{zhang2016accelerating,Guo2018Network,jaderberg2014speeding} attempted to decompose pre-trained models in order to get initial low-rank networks. However, the heuristically imposed low-rank could incur huge accuracy loss and network retraining is needed to recover the performance of the original network as much as possible. Some attempts were made to use sparsity regularization \cite{Wen2017Coordinating,chen2015compressing} to constrain the network into a low-rank space. Though sparsity regularization reduces the error incurred by decomposition to some extent, performance still degrades rapidly when compression rate increases.
	
	This paper is an extension of \cite{xu2018trained}. In this paper, we propose a new method, namely Trained Rank Pruning (TRP), for training low-rank networks. We embed the low-rank decomposition into the training process by gradually pushing the weight distribution of a well functioning network into a low-rank form, where all parameters of the original network are kept and optimized to maintain its capacity. We also propose a stochastic sub-gradient descent optimized nuclear regularization that further constrains the weights in a low-rank space to boost the TRP. 
	 The proposed solution is illustrated in Fig.~\ref{fig.1}.
	
	Overall, our contributions are summarized below.
	\begin{enumerate}
		\setlength\itemsep{-0em}
		\item A new training method called the TRP is presented by explicitly embedding the low-rank decomposition into the network training;
		\item A nuclear regularization is optimized by stochastic sub-gradient descent to boost the performance of the TRP;
		\item Improving inference acceleration and reducing approximation accuracy loss in both channel-wise and spatial-wise decomposition methods.  
	\end{enumerate}
	
	\begin{figure*}[h]
		\centering
		\includegraphics[width=5.5in]{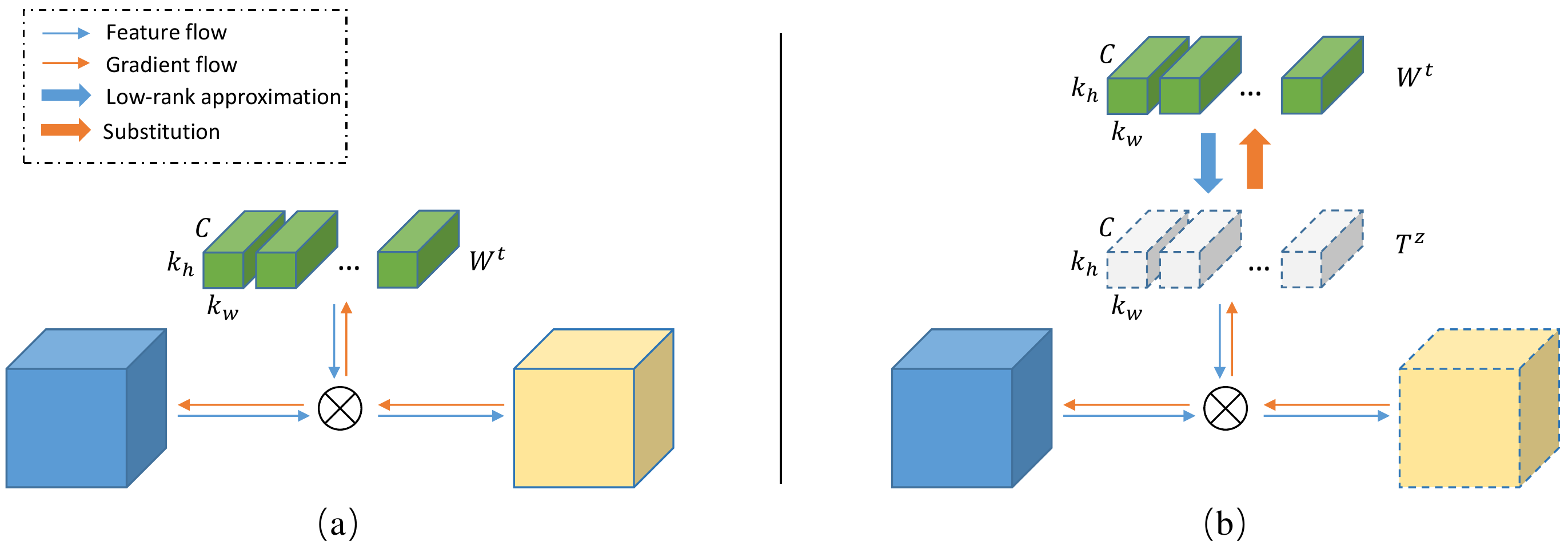}
		\caption{The training of TRP consists of two parts as illustrated in (a) and (b). (a) one normal iteration with forward-backward broadcast and weight update. (b) one training iteration inserted by TRP, where the low-rank approximation is first applied on filters before convolution. During backward propagation, the gradients are directly added on low-rank filters and the original weights are substituted by updated low-rank filters. (b) is applied once every $m$ iterations (\textit{i.e.} when gradient update iteration $t=zm, z=0, 1, 2, \cdots$), otherwise (a) is applied.}\label{fig.1}
	\end{figure*}
	\section{Related Works}
	
	A lot of works have been proposed to accelerate the inference process of deep neural networks. Briefly, these works could be categorized into three main categories: quantization, pruning, and low-rank decomposition.

	\textbf{Quantization} Weight quantization methods include training a quantized model from scratch \cite{chen2015compressing,Courbariaux2016BinaryNet,rastegari2016xnor} or converting a pre-trained model into quantized representation \cite{Zhou2016Incremental,Han2015DeepCC,Xu2018DeepNN}. The quantized weight representation includes binary value \cite{rastegari2016xnor,Courbariaux2016BinaryNet} or hash buckets \cite{chen2015compressing}. 
	Note that our method is inspired by the scheme of combining quantization with training process, \textit{i.e.} we embed the low-rank decomposition into training process to explicitly guide the parameter to a low-rank form.
	
	\textbf{Pruning} Non-structured and structured sparsity are introduced by pruning. \cite{han2015learning} proposes to prune unimportant connections between neural units with small weights in a pre-trained CNN. \cite{Wen2016LearningSS} utilizes group Lasso strategy to learn the structure sparsity of networks. \cite{Liu2017learning} adopts a similar strategy by explicitly imposing scaling factors on each channel to measure the importance of each connection and dropping those with small weights. In \cite{He_2017_ICCV}, the pruning problem is formulated as a data recovery problem. Pre-trained filters are re-weighted by minimizing a data recovery objective function. Channels with smaller weight are pruned. \cite{Luo2017ThiNetAF} heuristically selects filters using  change of next layer's output as a criterion.

	\textbf{Low-rank decomposition} Original models are decomposed into compact ones with more lightweight layers.
	\cite{jaderberg2014speeding} considers both the spatial-wise and channel-wise redundancy and proposes decomposing a filter into two cascaded asymmetric filters. \cite{zhang2016accelerating} further assumes the feature map lie in a low-rank subspace and decompose the convolution filter into $k\times k$ followed by $1\times 1$ filters via SVD. \cite{Guo2018Network} exploits the low-rank assumption of convolution filters and decompose a regular convolution into several depth-wise and point-wise convolution structures. Although these works achieved notable performance in network compression, all of them are based on the low-rank assumption. When such assumption is not completely satisfied, large prediction error may occur.
	
	Alternatively, some other works \cite{Wen2017Coordinating,Alvarez2017Compression} implicitly utilize sparsity regularization to direct the neural network training process to learn a low-rank representation. Our work is similar to this low-rank regularization method. However, in addition to appending an implicit regularization during training, we impose an explicit sparsity constraint in our training process and prove that our approach can push the weight distribution into a low-rank form quite effectively.
	
	\section{Methodology}
	\subsection{Preliminaries}
	Formally, the convolution filters in a layer can be denoted by a tensor $W \in \mathbb{R}^{n \times c \times k_w \times k_h}$, where $ n $ and $ c $ are the number of filters and input channels, $ k_h $ and $ k_w $ are the height and width of the filters. An input of the convolution layer $ F_i \in  \mathbb{R}^{c \times x \times y}$ generates an output as $ F_o = W * F_i $. Channel-wise correlation \cite{zhang2016accelerating} and spatial-wise correlation \cite{jaderberg2014speeding} are explored to approximate convolution filters in a low-rank space. In this paper, we focus on these two decomposition schemes. 
	However, unlike the previous works, we propose a new training scheme TRP to obtain a low-rank network without re-training after decomposition.

	\subsection{Trained Rank Pruning}
	Trained Rank Pruning (TRP) is motivated by the strategies of training quantized nets.
	One of the gradient update schemes to train quantized networks from scratch \cite{Li2017TrainingQN} is
	
	\begin{equation}\label{equ.1}
	w^{t+1}=Q(w^t-\alpha \triangledown f(w^t))
	\end{equation}
	where $Q(\cdot)$ is the quantization function, $w^t$ denote the parameter in the $t^{th}$ iteration. Parameters are quantized by $Q(\cdot)$ before updating the gradients.
	
	In contrast, we propose a simple yet effective training scheme called Trained Rank Pruning (TRP) in a periodic fashion:
	
	\begin{equation}\label{equ.2}
	\begin{split}
	W^{t+1}&=\left \{ 
	\begin{array}{c}
	W^t - \alpha \triangledown f(W^t)  \quad t \% m \neq 0 \\
	T^z - \alpha \triangledown f(T^z)  \quad t \% m = 0 
	\end{array} \right. \\
	T^{z}&=\mathcal{D}(W^{t}), \quad z = t / m
	\end{split}
	\end{equation}
	where $\mathcal{D}({\cdot})$ is a  low-rank  tensor approximation operator, $\alpha$ is the learning rate, $t$ indexes the iteration and $z$ is the iteration of the operator $\mathcal{D}$ , with $m$ being the period for the low-rank approximation.
	
	At first glance, this TRP looks very simple. An immediate concern arises: can the iterations guarantee the rank of the parameters converge, and more importantly would not increase when they are updated in this way? A positive answer (see Theorem 2) given in our theoretical analysis will certify the legitimacy of this algorithm.
	
	For the network quantization, if the gradients are smaller than the quantization, the gradient information would be totally lost and become zero. However, it will not happen in TRP because the low-rank operator is applied on the weight tensor. Furthermore, we apply low-rank approximation every $m$ SGD iterations. This saves training time to a large extend. As illustrated in Fig.~\ref{fig.1}, for every $m$ iterations, we perform low-rank approximation on the original filters, while gradients are updated on the resultant low-rank form. Otherwise, the network is updated via the normal SGD. 
	Our training scheme could be combined with any low-rank operators. In the proposed work, we choose the low-rank techniques proposed in \cite{jaderberg2014speeding} and \cite{zhang2016accelerating}, both of which transform the 4-dimensional filters into 2D matrix and then apply the truncated singular value decomposition (TSVD). The SVD of matrix ${W}^t$ can be written as:
	\begin{equation}\label{equ.3}
	W^t=\sum_{i=1}^{rank(W^t)}\sigma_i\cdot U_i\cdot (V_i)^T
	\end{equation}
	where $\sigma_i$ is the singular value of $W^t$ with $\sigma_1\geq \sigma_2 \geq \cdots \geq \sigma_{rank(W^t)}$, and $U_i$ and $V_i$ are the singular vectors. The parameterized TSVD($W^t;e$) is to find the smallest integer $k$ such that
	\begin{equation}\label{equ.4}
	\sum_{j=k+1}^{rank(W^t)}(\sigma_j)^2 \leq \quad e \sum_{i=1}^{rank(W^t)}(\sigma_i)^2
	\end{equation}
	where $e$ is a pre-defined hyper-parameter of the energy-pruning ratio, $e \in(0 , 1)$.

	After truncating the last $n-k$ singular values, we transform the low-rank 2D matrix  back to 4D tensor. Compared with directly training low-rank structures from scratch, the proposed TRP has following advantages.
	
	(1) Unlike updating the decomposed filters independently of the network training in literature \cite{zhang2016accelerating,jaderberg2014speeding}, we update the network directly on the original 4D shape of the decomposed parameters, which enable jointly network decomposition and training by preserving its discriminative capacity as much as possible.
	
	(2) Since the gradient update is  performed based on the original network structure, there will be no exploding and vanishing gradients problems caused by additional layers.
	
	(3) The rank of each layer is automatically selected during the training. We will prove a theorem certifying the rank of network weights convergence and would not increase in section \ref{sec:thm}.
	
	
	\subsection{Nuclear Norm Regularization}
	Nuclear norm is widely used in matrix completion problems. Recently, it is introduced to constrain the network into low-rank space during the training process \cite{Alvarez2017Compression}.
	\begin{equation}\label{equ.5}
	\min \left\{  f\left(x; w\right)+\lambda \sum_{l=1}^L||W_l||_* \right\}
	\end{equation}
	where $f(\cdot)$ is the objective loss function, nuclear norm $||W_l||_*$ is defined as $||W_l||_*=\sum_{i=1}^{rank(W_l)}\sigma_l^i$, with $\sigma_l^i$ the singular values of $W_l$. $\lambda$ is a hyper-parameter setting the influence of the nuclear norm. In \cite{Alvarez2017Compression} the proximity operator is applied in each layer independently to solve Eq.~(\ref{equ.5}). However, the proximity operator is split from the training process and doesn't consider the influence within layers.
	
	In this paper, we utilize stochastic sub-gradient descent \cite{Avron2012EfficientAP} to optimize nuclear norm regularization in the training process. Let $W=U\Sigma V^T$ be the SVD of $W$ and let $U_{tru}, V_{tru}$ be $U, V$ truncated to the first $rank(W)$ columns or rows, then $U_{tru}V_{tru}^T$ is the sub-gradient of $||W||_*$ \cite{watson1992characterization}. Thus, the sub-gradient of  Eq.~(\ref{equ.5}) in a layer is 
	\begin{equation}\label{equ.6}
	\triangledown f+\lambda U_{tru}V_{tru}^T
	\end{equation}
	
	The nuclear norm and loss function are optimized simultaneously during the training of the networks and can further be combined with the proposed TRP.
	
	\subsection{Theoretic Analysis}\label{sec:thm}
	In this section, we analyze the rank convergence of TRP from the perspective of matrix perturbation theory \cite{stewart1990matrix}. We prove that rank in TRP is monotonously decreasing, \textit{i.e.,} the model gradually converges to a more sparse model.
	
	Let A be an $m\times n$ matrix, without loss of generality, $m\geq n$. 
	$\Sigma = diag\left(\sigma_1,\cdots,\sigma_n\right)$ and $\sigma_1\geq \sigma_2\geq\cdots\geq\sigma_n$. $\Sigma$ is the diagonal matrix composed by all singular values of $A$. 
	Let $\widetilde{A}=A+E$ be a perturbation of $A$, and $E$ is the noise matrix. 
	$\widetilde{\Sigma} = diag\left(\widetilde{\sigma}_1,\cdots,\widetilde{\sigma}_n\right)$ and $\widetilde{\sigma}_1 \geq \widetilde{\sigma}_2 \geq \cdots \geq \widetilde{\sigma}_n$. $\widetilde{\sigma}_i$ is the singular values of $\widetilde{A}$.
	The basic perturbation bounds for the singular values of a matrix are given by
	
	\begin{theorem}\label{theorem1}
		Mirsky's theorem \cite{mirsky1960symmetric}:
		\begin{equation}\label{equ.9}
		\sqrt{\sum_{i}|\widetilde{\sigma}_i-\sigma_i|^2}\leq||E||_F\\
		\end{equation}
	\end{theorem}
	
	where $||\cdot||_F$ is the Frobenius norm. Then the following corollary can be inferred from Theorem \ref{theorem1}, 
	
	\begin{corollary}\label{corollary1}
		Let $B$ be any $m\times n$ matrix of rank not greater than $k$, \textit{i.e.} the singular values of B can be denoted by $\varphi_1 \geq \cdots \geq \varphi_k \geq 0$ and $\varphi_{k+1}=\cdots=\varphi_n=0$. Then
		\begin{equation}\label{equ.10}
		||B-A||_F \geq \sqrt{\sum_{i=1}^{n}|\varphi_i-\sigma_i|^2}\geq \sqrt{\sum_{j=k+1}^{n}\sigma_{j}^2}\\
		\end{equation}
	\end{corollary}
	
	Below, we will analyze the training procedure of the proposed TRP. Note that $W$ below are all transformed into 2D matrix. In terms of Eq.~(\ref{equ.2}), the training process between two successive TSVD operations can be rewritten as Eq. (\ref{equ.11})
	\begin{equation}\label{equ.11}
	\begin{split}
	W^t&=T^z=TSVD(W^t;e)\\
	W^{t+m} &= T^z-\alpha \Sigma_{i=0}^{m-1}\triangledown f(W^{t+i})\\
	T^{z+1}&=TSVD(W^{t+m};e)
	\end{split}
	\end{equation}
	where $W^t$ is the weight matrix in the $t$-th iteration. $T^z$ is the weight matrix  after applying TSVD over $W^t$. $\triangledown f(W^{t+i})$ is the gradient back-propagated during the $(t+i)$-th iteration. $e\in\left(0,1\right)$ is the predefined energy threshold. Then we have following theorem.
	\begin{theorem}\label{thm:2}
		Assume that $||\alpha \triangledown f||_F$ has an upper bound G, if  $G<\frac{\sqrt{e}}{m}||W^{t+m}||_F$, then $rank(T^z)\geq rank(T^{z+1})$. 
	\end{theorem}
	
	\begin{proof}

		We denote $\sigma^t_{j}$ and $\sigma^{t+m}_{j}$ as the singular values of $W^t$ and $W^{t+m}$ respectively. Then at the $t$-th iteration, given the energy ratio threshold $e$, the TSVD operation tries to find the singular value index $k \in [0, n-1]$ such that :
		
		\begin{equation}\label{equ.12}
		\begin{split}
		\sum_{j=k+1}^{n}\left(\sigma^{t}_{j}\right)^2&<e||W^t||_F^2\\
		\sum_{j=k}^{n}\left(\sigma^t_{j}\right)^2&\geq e||W^t||_F^2\\
		\end{split}
		\end{equation}
		
		In terms of Eq.~(\ref{equ.12}), $T^z$ is a $k$ rank matrix, i.e, the last $n-k$ singular values of $T^z$ are equal to $0$. According to Corollary \ref{corollary1}, we can derive that:
		\begin{equation}\label{equ.13}
		\begin{split}
		||W^{t+m}-T^z||_F& = ||\alpha \sum_{i=0}^{m-1}\triangledown f^{t+i}||_F\\
		&\geq \sqrt{\sum_{j=k+1}^{n}\left(\sigma^{t+m}_{j}\right)^2} \\
		\end{split}
		\end{equation}
		
		Given the assumption $G < \frac{\sqrt{e}}{m} ||W^{t+m}||_F $, we can get:
		\begin{equation}\label{equ.14}
		\begin{split}
		\frac{\sqrt{\sum_{j=k+1}^{n}\left(\sigma^{t+m}_{j}\right)^2}}{||W^{t+m}||_F}&\leq \frac{||\alpha \sum_{i=0}^{m-1}\triangledown f^{t+i}||_F}{||W^{t+m}||_F}\\
		&\leq \frac{\sum_{i=0}^{m-1}||\alpha \triangledown f^{t+i}||_F}{||W^{t+m}||_F}\\
		&\leq \frac{mG}{||W^{t+m}||_F}< \sqrt{e} \\
		\end{split}
		\end{equation}

		
		Eq.~(\ref{equ.14}) indicates that 
		since the perturbations of singular values are bounded by the parameter gradients, if we properly select the TSVD energy ratio threshold $e$, we could guarantee that if $n-k$ singular values are pruned by previous TSVD iteration, then before the next TSVD, the energy for the last $n-k$ singular values is still less than the pre-defined energy threshold $e$. Thus TSVD should keep the number of pruned singular values or drop more to achieve the criterion in Eq.~(\ref{equ.12}), consequently a weight matrix with lower or same rank is obtained, \textit{i.e.} $Rank(T^{z})\geq Rank(T^{z+1})$. We further confirm our analysis about the variation of rank distribution in Section \ref{sec:exp}.
	\end{proof}
	
	\begin{table}[htb]
		\begin{center}
			\begin{tabular}{|l|c|c|c|}
				\hline
				Model & Top 1 ($\%$) & Speed up\\
				\hline\hline
				R-20 (baseline)&91.74&1.00$\times$\\
				\hline
				R-20 (TRP1)&90.12&1.97$\times$\\
				R-20 (TRP1+Nu)&\textbf{90.50}&\textbf{2.17$\times$}\\
				R-20 (\cite{zhang2016accelerating})&88.13&1.41$\times$\\
				\hline
				R-20 (TRP2)&90.13&2.66$\times$\\
				R-20 (TRP2+Nu)&\textbf{90.62}&\textbf{2.84$\times$}\\
				R-20 (\cite{jaderberg2014speeding})&89.49&1.66$\times$\\
				\hline
				\hline
				R-56 (baseline)&93.14&1.00$\times$\\
				\hline
				R-56 (TRP1)&\textbf{92.77}&2.31$\times$\\
				R-56 (TRP1+Nu)&91.85&\textbf{4.48$\times$}\\
				R-56 (\cite{zhang2016accelerating})&91.56&2.10$\times$\\
				\hline
				R-56 (TRP2)&\textbf{92.63}&2.43$\times$\\
				R-56 (TRP2+Nu)&91.62&\textbf{4.51$\times$}\\
				R-56 (\cite{jaderberg2014speeding})&91.59&2.10$\times$\\
				\hline
				R-56 \cite{He_2017_ICCV}&91.80&2.00$\times$\\
				R-56 \cite{li2016pruning}&91.60&2.00$\times$\\
				\hline
			\end{tabular}
		\end{center}
		\caption{Experiment results on CIFAR-10. "R-`` indicates ResNet-. }\label{tab1}\vspace*{-0.4cm}
	\end{table}
	
	\section{Experiments}\label{sec:exp}
	
	\subsection{Datasets and Baseline}
	We evaluate the performance of TRP scheme on two common datasets, CIFAR-10 \cite{AlexCifar10} and ImageNet \cite{Deng2009ImageNetAL}. The CIFAR-10 dataset consists of  colored natural images with $32\times 32$
	resolution and has totally 10 classes. The ImageNet dataset consists of 1000 classes of images for recognition task. 
	For both of the datasets, we adopt ResNet \cite{He2016DeepRL} as our baseline model since it is widely used in different vision tasks.
	We use ResNet-20, ResNet-56 for CIFAR-10 and ResNet-18, ResNet-50 for ImageNet. For evaluation metric, we adopt top-1 accuracy on CIFAR-10 and top-1, top-5 accuracy on ImageNet. To measure the acceleration performance, we compute the FLOPs ratio between baseline and decomposed models to obtain the final speedup rate. Wall-clock CPU and GPU time is also compared. Apart from the basic decomposition methods, we compare the performance with other state-of-the-art acceleration algorithms~\cite{He_2017_ICCV,li2016pruning,Luo2017ThiNetAF,Zhou_2019_ICCV}.
	
	\begin{table}[htb]
		\begin{center}
			\begin{tabular}{|l|c|c|c|}
				\hline
				Method &Top1($\%$)& Top5($\%$)& Speed up\\
				\hline\hline
				Baseline&69.10&88.94&1.00$\times$\\
				\hline
				TRP1 &\textbf{65.46}&\textbf{86.48} &1.81$\times$\\
				TRP1+Nu&65.39&86.37&\textbf{2.23$\times$}\\
				\cite{zhang2016accelerating}\footnotemark[1] &-& 83.69&1.39$\times$ \\
				\cite{zhang2016accelerating}&63.10&84.44&1.41$\times$\\
				\hline
				TRP2 &\textbf{65.51}&\textbf{86.74} &$2.60\times$\\
				TRP2+Nu&65.34&86.61&\textbf{3.18}$\times$\\
				\cite{jaderberg2014speeding}&62.80&83.72&2.00$\times$\\
				\hline
			\end{tabular}
		\end{center}
		\caption{Results of ResNet-18 on ImageNet. }\label{tab2}\vspace*{-0.2cm}
	\end{table}

	\begin{table}[htb]
		\begin{center}
			\begin{tabular}{|l|c|c|c|}
				\hline
				Method &Top1($\%$)& Top5($\%$) & Speed up\\
				\hline\hline
				Baseline&75.90&92.70&1.00$\times$\\
				\hline
				TRP1+Nu&72.69&91.41&\textbf{2.30}$\times$\\
				TRP1+Nu&\textbf{74.06}&\textbf{92.07}&1.80$\times$\\
				\cite{zhang2016accelerating}&71.80&90.2&1.50$\times$\\
				\cite{He_2017_ICCV}&-&90.80&2.00\\
				\cite{Luo2017ThiNetAF}&72.04&90.67&1.58\\
				\cite{luo2018thinet}&72.03&90.99&2.26\\
				\cite{Zhou_2019_ICCV}&71.50&90.20&2.30\\
				\hline
			\end{tabular}
		\end{center}
		\caption{Results of ResNet-50 on ImageNet.}\label{tab3}\vspace*{-0.4cm}
	\end{table}
	
	\subsection{Implementation Details}
	We implement our TRP scheme with NVIDIA 1080 Ti GPUs. For training on CIFAR-10, we start with base learning rate of $0.1$ to train 164 epochs and degrade the value by a factor of $10$ at the $82$-th and $122$-th epoch. For ImageNet, we directly finetune the model with TRP scheme from the pre-trained baseline with learning rate $0.0001$ for 10 epochs. We adopt SGD solver to update weight and set the weight decay value as $10^{-4}$ and momentum value as $0.9$. The accuracy improvement enabled by data dependent decomposition vanishes after fine-tuning. So we simply
	adopt the retrained data independent decomposition as our basic methods.
	\footnotetext[1]{the implementation of \cite{Guo2018Network}} 
	
	\subsection{Results on CIFAR-10}
	\textbf{Settings.} Experiments on channel-wise decomposition (TRP1) and spatial-wise decomposition (TRP2) are both considered. The TSVD energy threshold in TRP and TRP+Nu is $0.02$ and the nuclear norm weight $\lambda$ is set as $0.0003$. We decompose both the $1\times 1$ and $3\times 3$ layers in ResNet-56. 
	
	\textbf{Results.} As shown in Table~\ref{tab1}, for both spatial-wise and channel-wise decomposition, the proposed TRP outperforms basic methods~\cite{zhang2016accelerating,jaderberg2014speeding} on ResNet-20 and ResNet-56. Results become even better when nuclear regularization is used. For example, in the channel-wise decomposition (TRP2) of ResNet-56, results of TRP combined with nuclear regularization can even achieve $2\times$ speed up rate than \cite{zhang2016accelerating} with same accuracy drop. TRP also outperforms filter pruning~\cite{li2016pruning} and channel pruning~\cite{He_2017_ICCV}. The channel decomposed TRP trained ResNet-56 can achieve $92.77\%$ accuracy with $2.31\times$ acceleration, while \cite{He_2017_ICCV} is $91.80\%$ and \cite{li2016pruning} is $91.60\%$. With nuclear regularization, our methods can approximately double the acceleration rate of \cite{He_2017_ICCV} and \cite{li2016pruning} with higher accuracy.
	
	\begin{figure}[t]
		\centering
		\includegraphics[width=3in]{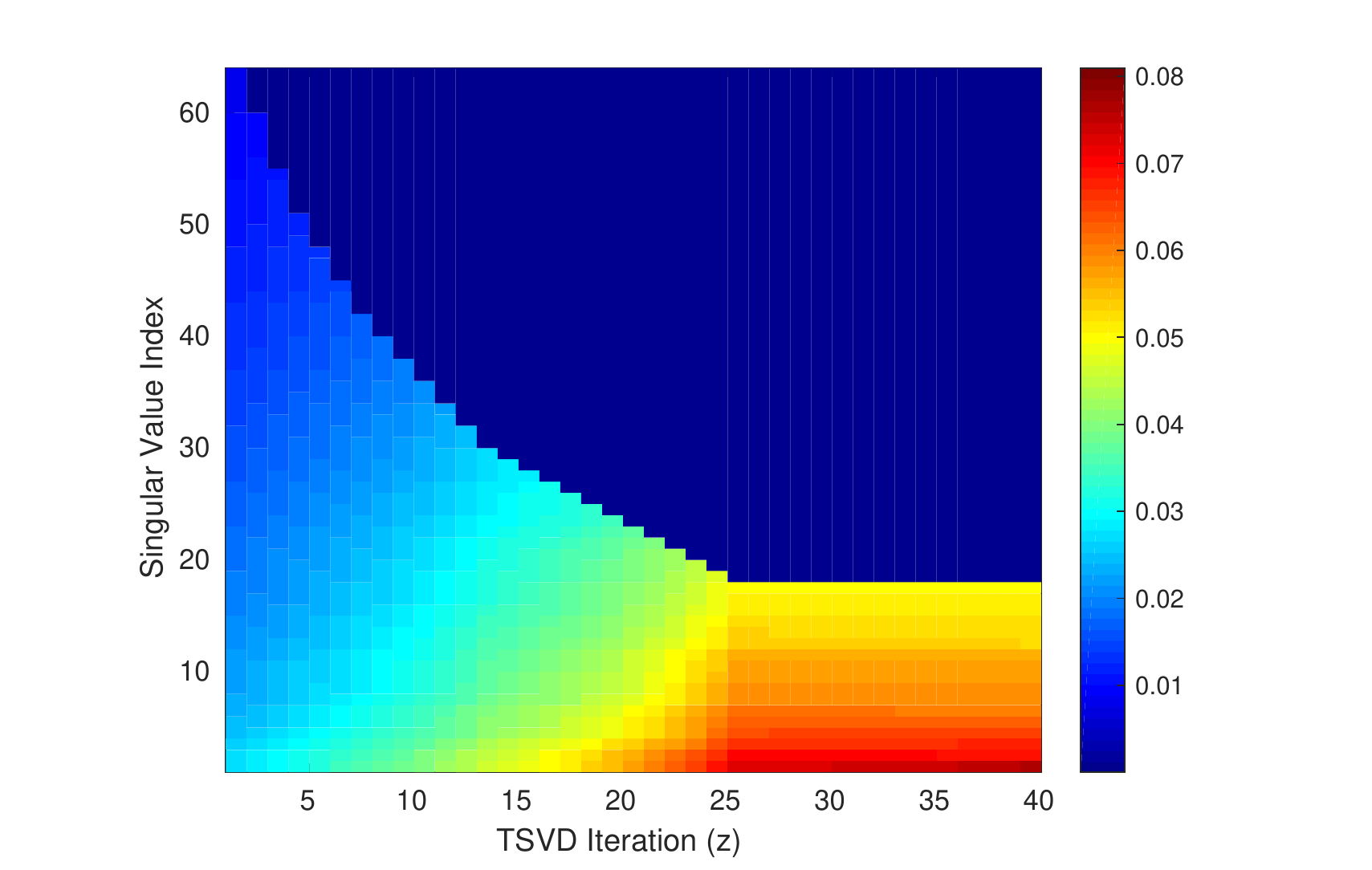}\\
		\caption{Visualization of rank selection, taken from the res3-1-2 convolution layer in ResNet-20 trained on CIFAR-10.}\label{fig.5}
		\vspace*{-0.4cm}
	\end{figure}

	\subsection{Results on ImageNet}
	\textbf{Settings.} We choose ResNet-18 and ResNet-50 as our baseline models. The TSVD energy threshold $e$ is set as 0.005. $\lambda$ of nuclear norm regularization is 0.0003 for both ResNet-18 and ResNet-50. We decompose both the $3\times 3$ and $1\times 1$ Convolution layers in ResNet-50. TRP1 is the channel-wise decomposition and TRP2 is the spatial-wise decomposition.
		
	\begin{figure*}[htb]
		\centering
		\subfigure[Channel-wise decomposition]{
			\includegraphics[width=2.8in]{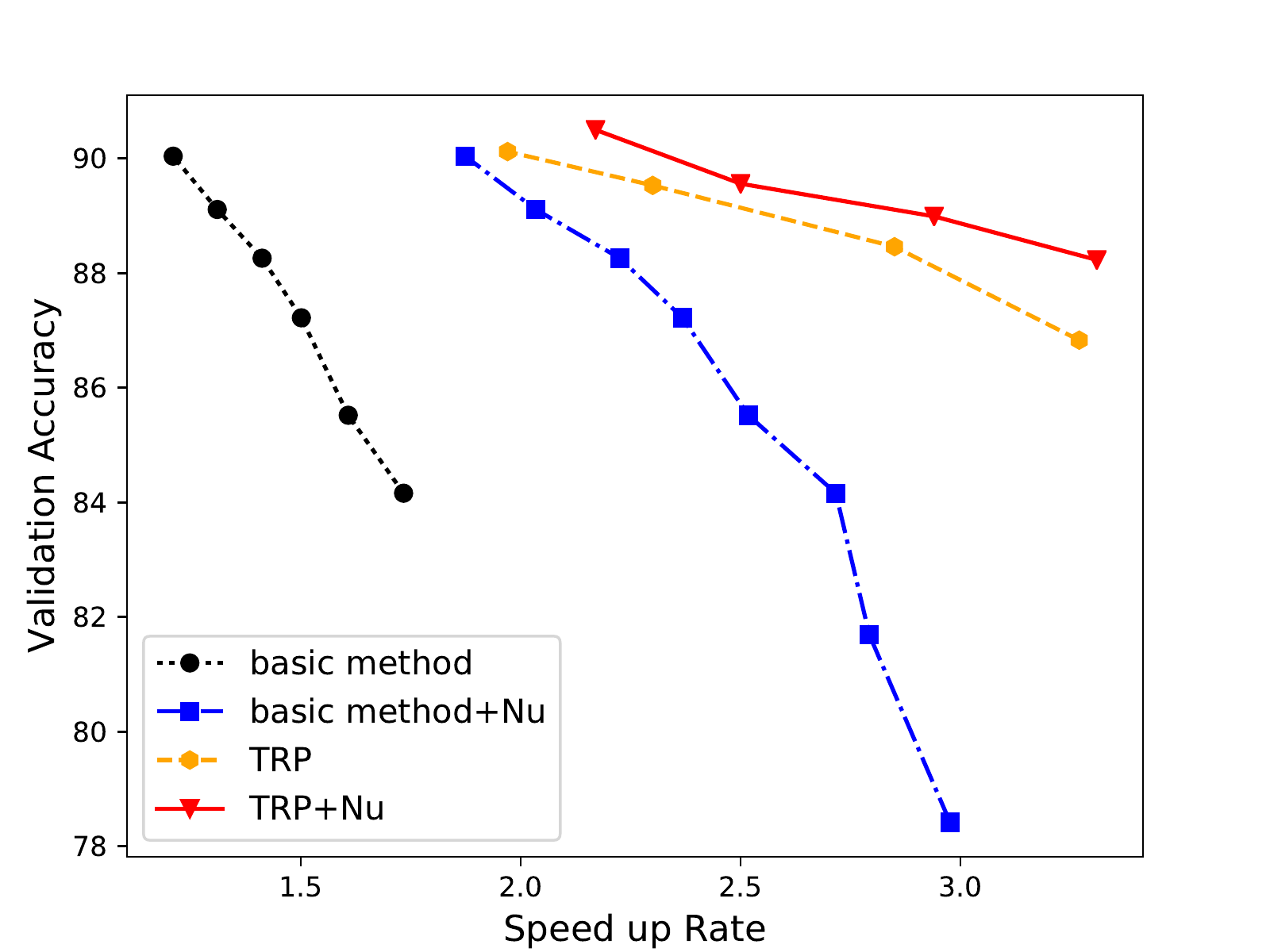}}
		\subfigure[Spatial-wise decomposition]{
			\includegraphics[width=2.8in]{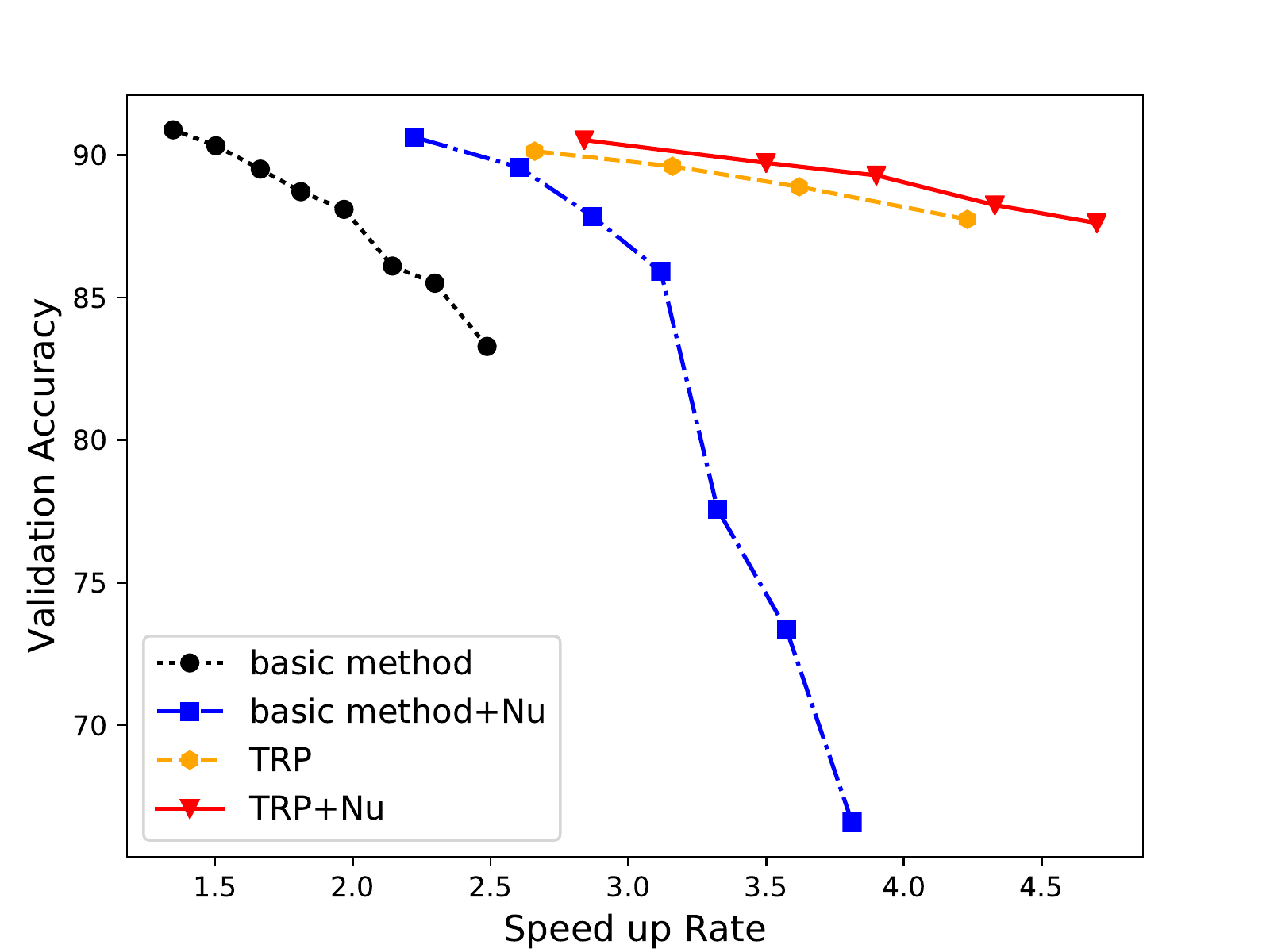}}
		\caption{Ablation study on ResNet-20. Basic methods are data-independent decomposition methods (channel or spatial) with finetuning.}
		\label{fig.6}\vspace*{-0.4cm}
	\end{figure*}

	\textbf{Results.} The results on ImageNet are shown in Table~\ref{tab2} and Table~\ref{tab3}. For ResNet-18, our method outperforms the basic methods~\cite{zhang2016accelerating,jaderberg2014speeding}. For example, in the channel-wise decomposition, TRP obtains 1.81$\times$ speed up rate with 86.48\% Top5 accuracy on ImageNet which outperforms both the data-driven~\cite{zhang2016accelerating}\textsuperscript{1} and data independent~\cite{zhang2016accelerating} methods by a large margin. Nuclear regularization can increase the speed up rates with the same accuracy.
	
	For ResNet-50, to better validate the effectiveness of our method, we also compare TRP with pruning based methods. With $1.80\times$ speed up, our decomposed ResNet-50 can obtain $74.06\%$ Top1 and $92.07\%$ Top5 accuracy which is much higher than \cite{Luo2017ThiNetAF}. The TRP achieves $2.23\times$ acceleration which is higher than \cite{He_2017_ICCV} with the same $1.4\%$ Top5 degrade. Besides, with the same $2.30\times$ acceleration rate, our performance is better than \cite{Zhou_2019_ICCV}.
	
	\subsection{Rank Variation}
	
	To analyze the variation of rank distribution during training, we further conduct an experiment on the CIFAR-10 dataset with ResNet-20 and extract the weight from the \emph{res3-1-2} convolution layer with channel-wise decomposition as our TRP scheme. After each TSVD, we compute the normalized energy ratio $ER(i)$ for each singular value $\sigma_i$ as Eq. (\ref{eq:energy}).
	
	\begin{equation}\label{eq:energy}
	ER(i) = \frac{\sigma_i^2}{\sum_{j=0}^{rank(T^z)}\sigma_j^2}
	\end{equation}
	we record for totally $40$ iterations of TSVD with period $m=20$, which is equal to $800$ training iterations, and our energy theshold $e$ is pre-defined as $0.05$. 
	Then we visualize the variation of $ER$ in Fig.~\ref{fig.5}. During our training, we observe that the theoretic bound value $\max_t\frac{mG}{||W^t||_F} \approx 0.092 < \sqrt{e} \approx 0.223$, which indicates that our basic assumption in theorem \ref{thm:2} always holds for the initial training stage.

	And this phenomenon is also reflected in Fig.~\ref{fig.5}, at the beginning, the energy distribution is almost uniform w.r.t each singular value, and the number of dropped singular values increases after each TSVD iteration 
	and the energy distribution becomes more dense among singular values with smaller index. Finally, the rank distribution converges to a certain point where the smallest energy ratio exactly reaches our threshold $e$ and TSVD will not cut more singular values.

	\subsection{Ablation Study}

	In order to show the effectiveness of different components of our method, we compare four training schemes, basic methods \cite{zhang2016accelerating,jaderberg2014speeding}, basic methods combined with nuclear norm regularization, TRP and TRP combined with nuclear norm regularization. The results are shown in Fig.~\ref{fig.6}. We can have following observations:
	
	(1) \emph{Nuclear norm regularization} After combining nuclear norm regularization, basic methods improve by a large margin. Since Nuclear norm regularization constrains the filters into low rank space, the loss caused by TSVD is smaller than the basic methods.
	
	(2) \emph{Trained rank pruning} As depicted in Fig.~\ref{fig.6}, when the speed up rate increases, the performance of basic methods and basic methods combined with nuclear norm regularization degrades sharply. However, the proposed TRP degrades very slowly. This indicates that by reusing the capacity of the network, TRP can learn a better low-rank feature representations than basic methods. The gain of nuclear norm regularization on TRP is not as big as basic methods because TRP has already induced the parameters into low-rank space by embedding TSVD in training process.
	\begin{table}[htb]
		\begin{center}
			\begin{tabular}{|l|c|c|}
				\hline
				Model & GPU time (ms)&CPU time (ms)\\
				\hline\hline
				Baseline&0.45&118.02\\
				\hline
				TRP1+Nu (channel)&0.33&64.75\\
				\hline
				TRP2+Nu (spatial)&0.31&49.88\\
				\hline
			\end{tabular}
		\end{center}
		\caption{Actual inference time per image on ResNet-18.}\label{tab4}
	\end{table}
	\vspace*{-0.2cm}
	\subsection{Runtime Speed up of Decomposed Networks}
	We further evaluate the actual runtime speed up of the compressed Network as shown in Table \ref{tab4}. Our experiment is conducted on a platform with one Nvidia 1080Ti GPU and Xeon E5-2630 CPU. The models we used are the original ResNet-18 and decomposed models by TRP1+Nu and TRP2+Nu. 
	From the results, we observe that on CPU our TRP scheme achieves more salient acceleration performance. Overall the spatial decomposition combined with our TRP+Nu scheme has better performance. Because cuDNN is not friendly for $1\times 3$ and $3\times1$ kernels, the actual speed up of spatial-wise decomposition is not as obvious as the reduction of FLOPs.
	
	\section{Conclusion}
	
	In this paper, we proposed a new scheme Trained Rank Pruning (TRP) for training low-rank networks. It leverages capacity and structure of the original network by embedding the low-rank approximation in the training process. Furthermore, we propose stochastic sub-gradient descent optimized nuclear norm regularization to boost the TRP. The proposed TRP can be incorporated with any low-rank decomposition method. On CIFAR-10 and ImageNet datasets, we have shown that our methods can outperform basic methods and other pruning based methods both in channel-wise decmposition and spatial-wise decomposition.
	\clearpage
	\bibliographystyle{named}
	\bibliography{ijcai20}

\begin{thebibliography}{}

\bibitem[\protect\citeauthoryear{Alvarez and
  Salzmann}{2017}]{Alvarez2017Compression}
Jose~M Alvarez and Mathieu Salzmann.
\newblock Compression-aware training of deep networks.
\newblock In {\em NIPS}, 2017.

\bibitem[\protect\citeauthoryear{Avron \bgroup \em et al.\egroup
  }{2012}]{Avron2012EfficientAP}
Haim Avron, Satyen Kale, Shiva~Prasad Kasiviswanathan, and Vikas Sindhwani.
\newblock Efficient and practical stochastic subgradient descent for nuclear
  norm regularization.
\newblock In {\em ICML}, 2012.

\bibitem[\protect\citeauthoryear{Chen \bgroup \em et al.\egroup
  }{2015}]{chen2015compressing}
Wenlin Chen, James Wilson, Stephen Tyree, Kilian Weinberger, and Yixin Chen.
\newblock Compressing neural networks with the hashing trick.
\newblock In {\em ICML}, 2015.

\bibitem[\protect\citeauthoryear{Courbariaux and
  Bengio}{2016}]{Courbariaux2016BinaryNet}
Matthieu Courbariaux and Yoshua Bengio.
\newblock Binarynet: Training deep neural networks with weights and activations
  constrained to +1 or -1.
\newblock {\em arXiv preprint arXiv:1602.02830}, 2016.

\bibitem[\protect\citeauthoryear{Deng \bgroup \em et al.\egroup
  }{2009}]{Deng2009ImageNetAL}
Jia Deng, Wei Dong, Richard Socher, Li-Jia Li, Kai Li, and Li~Fei-Fei.
\newblock Imagenet: A large-scale hierarchical image database.
\newblock {\em CVPR}, 2009.

\bibitem[\protect\citeauthoryear{Denton \bgroup \em et al.\egroup
  }{2014}]{Denton2014Exploiting}
Emily Denton, Wojciech Zaremba, Joan Bruna, Yann Lecun, and Rob Fergus.
\newblock Exploiting linear structure within convolutional networks for
  efficient evaluation.
\newblock In {\em NIPS}, 2014.

\bibitem[\protect\citeauthoryear{Guo \bgroup \em et al.\egroup
  }{2018}]{Guo2018Network}
Jianbo Guo, Yuxi Li, Weiyao Lin, Yurong Chen, and Jianguo Li.
\newblock Network decoupling: From regular to depthwise separable convolutions.
\newblock In {\em BMVC}, 2018.

\bibitem[\protect\citeauthoryear{Han \bgroup \em et al.\egroup
  }{2015a}]{Han2015DeepCC}
Song Han, Huizi Mao, and William~J. Dally.
\newblock Deep compression: Compressing deep neural network with pruning,
  trained quantization and huffman coding.
\newblock {\em CoRR}, abs/1510.00149, 2015.

\bibitem[\protect\citeauthoryear{Han \bgroup \em et al.\egroup
  }{2015b}]{han2015learning}
Song Han, Jeff Pool, John Tran, and William Dally.
\newblock Learning both weights and connections for efficient neural network.
\newblock In {\em NIPS}, 2015.

\bibitem[\protect\citeauthoryear{He \bgroup \em et al.\egroup
  }{2016}]{He2016DeepRL}
Kaiming He, Xiangyu Zhang, Shaoqing Ren, and Jian Sun.
\newblock Deep residual learning for image recognition.
\newblock 2016.

\bibitem[\protect\citeauthoryear{He \bgroup \em et al.\egroup
  }{2017}]{He_2017_ICCV}
Yihui He, Xiangyu Zhang, and Jian Sun.
\newblock Channel pruning for accelerating very deep neural networks.
\newblock In {\em ICCV}, 2017.

\bibitem[\protect\citeauthoryear{Jaderberg \bgroup \em et al.\egroup
  }{2014}]{jaderberg2014speeding}
Max Jaderberg, Andrea Vedaldi, and Andrew Zisserman.
\newblock Speeding up convolutional neural networks with low rank expansions.
\newblock {\em arXiv preprint arXiv:1405.3866}, 2014.

\bibitem[\protect\citeauthoryear{Krizhevsky and Hinton}{2009}]{AlexCifar10}
Alex Krizhevsky and Geoffrey Hinton.
\newblock Learning multiple layers of features from tiny images.
\newblock {\em Computer Science}, 2009.

\bibitem[\protect\citeauthoryear{Li \bgroup \em et al.\egroup
  }{2016}]{li2016pruning}
Hao Li, Asim Kadav, Igor Durdanovic, Hanan Samet, and Hans~Peter Graf.
\newblock Pruning filters for efficient convnets.
\newblock {\em arXiv preprint arXiv:1608.08710}, 2016.

\bibitem[\protect\citeauthoryear{Li \bgroup \em et al.\egroup
  }{2017}]{Li2017TrainingQN}
Hao Li, Soham De, Zheng Xu, Christoph Studer, Hanan Samet, and Tom Goldstein.
\newblock Training quantized nets: A deeper understanding.
\newblock In {\em NIPS}, 2017.

\bibitem[\protect\citeauthoryear{Liu \bgroup \em et al.\egroup
  }{2017}]{Liu2017learning}
Zhuang Liu, Jianguo Li, Zhiqiang Shen, Gao Huang, Shoumeng Yan, and Changshui
  Zhang.
\newblock Learning efficient convolutional networks through network slimming.
\newblock In {\em ICCV}, 2017.

\bibitem[\protect\citeauthoryear{Luo \bgroup \em et al.\egroup
  }{2017}]{Luo2017ThiNetAF}
Jian-Hao Luo, Jianxin Wu, and Weiyao Lin.
\newblock Thinet: A filter level pruning method for deep neural network
  compression.
\newblock {\em ICCV}, 2017.

\bibitem[\protect\citeauthoryear{Luo \bgroup \em et al.\egroup
  }{2018}]{luo2018thinet}
Jian-Hao Luo, Hao Zhang, Hong-Yu Zhou, Chen-Wei Xie, Jianxin Wu, and Weiyao
  Lin.
\newblock Thinet: pruning cnn filters for a thinner net.
\newblock {\em TPAMI}, 2018.

\bibitem[\protect\citeauthoryear{Ma \bgroup \em et al.\egroup
  }{2018}]{Ma2018ShuffleNet}
Ningning Ma, Xiangyu Zhang, Hai-Tao Zheng, and Jian Sun.
\newblock Shufflenet v2: Practical guidelines for efficient cnn architecture
  design.
\newblock {\em arXiv preprint arXiv:1807.11164}, 2018.

\bibitem[\protect\citeauthoryear{Mirsky}{1960}]{mirsky1960symmetric}
Leon Mirsky.
\newblock Symmetric gauge functions and unitarily invariant norms.
\newblock {\em The quarterly journal of mathematics}, 11(1):50--59, 1960.

\bibitem[\protect\citeauthoryear{Rastegari \bgroup \em et al.\egroup
  }{2016}]{rastegari2016xnor}
Mohammad Rastegari, Vicente Ordonez, Joseph Redmon, and Ali Farhadi.
\newblock Xnor-net: Imagenet classification using binary convolutional neural
  networks.
\newblock In {\em ECCV}, 2016.

\bibitem[\protect\citeauthoryear{Sandler \bgroup \em et al.\egroup
  }{2018}]{Sandler2018MobileNetV2}
Mark Sandler, Andrew Howard, Menglong Zhu, Andrey Zhmoginov, and Liang-Chieh
  Chen.
\newblock Mobilenetv2: Inverted residuals and linear bottlenecks.
\newblock In {\em CVPR}, June 2018.

\bibitem[\protect\citeauthoryear{Stewart}{1990}]{stewart1990matrix}
Gilbert~W Stewart.
\newblock Matrix perturbation theory.
\newblock 1990.

\bibitem[\protect\citeauthoryear{Watson}{1992}]{watson1992characterization}
G~Alistair Watson.
\newblock Characterization of the subdifferential of some matrix norms.
\newblock {\em Linear algebra and its applications}, 170:33--45, 1992.

\bibitem[\protect\citeauthoryear{Wen \bgroup \em et al.\egroup
  }{2016}]{Wen2016LearningSS}
Wei Wen, Chunpeng Wu, Yandan Wang, Yiran Chen, and Hai Li.
\newblock Learning structured sparsity in deep neural networks.
\newblock In {\em NIPS}, 2016.

\bibitem[\protect\citeauthoryear{Wen \bgroup \em et al.\egroup
  }{2017}]{Wen2017Coordinating}
Wei Wen, Cong Xu, Chunpeng Wu, Yandan Wang, Yiran Chen, and Hai Li.
\newblock Coordinating filters for faster deep neural networks.
\newblock In {\em ICCV}, 2017.

\bibitem[\protect\citeauthoryear{Xu \bgroup \em et al.\egroup
  }{2018}]{Xu2018DeepNN}
Yuhui Xu, Yongzhuang Wang, Aojun Zhou, Weiyao Lin, and Hongkai Xiong.
\newblock Deep neural network compression with single and multiple level
  quantization.
\newblock {\em CoRR}, abs/1803.03289, 2018.

\bibitem[\protect\citeauthoryear{Xu \bgroup \em et al.\egroup
  }{2019}]{xu2018trained}
Yuhui Xu, Yuxi Li, Shuai Zhang, Wei Wen, Botao Wang, Yingyong Qi, Yiran Chen,
  Weiyao Lin, and Hongkai Xiong.
\newblock Trained rank pruning for efficient deep neural networks.
\newblock In {\em NIPS EMC2 workshop}, 2019.

\bibitem[\protect\citeauthoryear{Zhang \bgroup \em et al.\egroup
  }{2016}]{zhang2016accelerating}
Xiangyu Zhang, Jianhua Zou, Kaiming He, and Jian Sun.
\newblock Accelerating very deep convolutional networks for classification and
  detection.
\newblock {\em TPAMI}, 38(10):1943--1955, 2016.

\bibitem[\protect\citeauthoryear{Zhou \bgroup \em et al.\egroup
  }{2017}]{Zhou2016Incremental}
Aojun Zhou, Anbang Yao, Yiwen Guo, Lin Xu, and Yurong Chen.
\newblock Incremental network quantization: Towards lossless cnns with
  low-precision weights.
\newblock {\em arXiv preprint arXiv:1702.03044}, 2017.

\bibitem[\protect\citeauthoryear{Zhou \bgroup \em et al.\egroup
  }{2019}]{Zhou_2019_ICCV}
Yuefu Zhou, Ya~Zhang, Yanfeng Wang, and Qi~Tian.
\newblock Accelerate cnn via recursive bayesian pruning.
\newblock In {\em ICCV}, 2019.

\end{thebibliography}
	\appendix

\end{document}